\documentclass[11pt]{article}
\usepackage{microtype}
\usepackage{graphicx}
\usepackage{caption,subcaption}
\usepackage{amsmath,amsbsy,amsfonts,amssymb,amsthm,bm}
\usepackage{algorithm,algorithmic,mathtools}
\usepackage{color,cases,multirow}
\usepackage[margin=1.2in]{geometry} 
\usepackage[colorlinks=True, citecolor=blue]{hyperref}
\usepackage{todonotes}
\usepackage{authblk}

\usepackage{amsmath,amsbsy,amsfonts,amssymb,amsthm,bm,color}
\newtheorem{theorem}{Theorem}
\newtheorem{remark}{Remark}

\theoremstyle{definition}
\newtheorem{definition}{Definition}

\newcommand{\bR}{\mathbb{R}}
\newcommand{\bE}{\mathbb{E}}
\newcommand{\bx}{\mathbf{x}}
\newcommand{\bg}{\mathbf{g}}
\newcommand{\be}{\mathbf{e}}
\newcommand{\cA}{\mathcal{A}}
\newcommand{\cF}{\mathcal{F}}
\newcommand{\cD}{\mathcal{D}}
\newcommand{\cT}{\mathcal{T}}
\newcommand{\cP}{\mathcal{P}}
\newcommand{\bP}{\mathbb{P}}
\newcommand{\stat}{\textrm{STAT}}
\newcommand{\supp}{\textrm{supp}}

\usepackage{tikz}
\usetikzlibrary{shapes.geometric, arrows}
\tikzstyle{generator} = [circle, text centered, draw=black, fill=red!30]
\tikzstyle{model} = [rectangle, rounded corners, minimum width=3cm, minimum height=1cm,text centered, draw=black, fill=red!30]
\tikzstyle{arrow} = [thick,->,>=stealth]

\title{Achieving Adversarial Robustness Requires An Active Teacher}

\author[1]{Chao Ma \thanks{\texttt{chaoma@stanford.edu}}}
\author[1]{Lexing Ying \thanks{\texttt{lexing@stanford.edu}}}

\affil[1]{Department of Mathematics, Stanford University}

\begin{document}

\maketitle

\begin{abstract}
  A new understanding of adversarial examples and adversarial robustness is proposed by decoupling the data generator and the label generator (which we call the teacher). In our framework, adversarial robustness is a conditional concept---the student model is not absolutely robust, but robust with respect to the teacher. Based on the new understanding, we claim that adversarial examples exist because the student cannot obtain sufficient information of the teacher from the training data. Various ways of achieving robustness is compared. Theoretical and numerical evidence shows that to efficiently attain robustness, a teacher that actively provides its information to the student may be necessary. 
\end{abstract}

\section{Introduction}

The existence of adversarial examples restricts the application of deep learning in many fields with high demand on the robustness and security, such as autonomous driving and health care. Hence, improving adversarial robustness of deep neural networks has experienced extensive study, both theoretically and practically~\cite{akhtar2018threat,hao2020adversarial}. Originally, adversarial examples are found to be perturbed images whose perturbations are imperceptible to humans but cause huge error to the neural networks~\cite{szegedy2013intriguing,biggio2013evasion}. In most existing works, however, adversarial robustness is defined as robustness with respect to perturbations measured by the $l_p$ distance (e.g. \cite{szegedy2013intriguing,goodfellow2014explaining}). Specifically, a model $f_\theta(\cdot)$ is considered to be robust if the adversarial loss
\begin{equation}
    L_{\textrm{adv}}(f_\theta) = \bE_{(\bx, y)} \max_{\|\delta\|_p\leq\epsilon} l(f_\theta(\bx+\delta), y)
\end{equation}
is small, where $\epsilon$ is a pre-defined value and $l$ is some loss function~\cite{madry2017towards}. This simplification helps analysis and implementation. In spite of this, the robustness with small $l_p$ perturbations is very different from the robustness with respect to human-imperceptible perturbations~\cite{sen2019should}. A human-imperceptible perturbation may not have small $l_p$ norm~\cite{engstrom2019exploring,zhao2020towards}, and a perturbation with small $l_p$ norm may also not necessarily be imperceptible to humans~\cite{sharif2018suitability}. In Figure~\ref{fig: human_diff_lp}, inspired by optical illusions, we show an example of difference between some $l_p$ distances and human perception. This difference makes current ``adversarially robust'' models easily broken by newly-designed attacks. Besides $l_p$ distances, other measures, such as Wasserstein distance~\cite{wong2019wasserstein} and structural similarity (SSIM)~\cite{wang2004image}, are also shown to be different from human perception~\cite{sen2019should}. 

\begin{figure}
    \centering
    \begin{subfigure}{0.32\textwidth}
    \centering
    \includegraphics[width=0.95\textwidth]{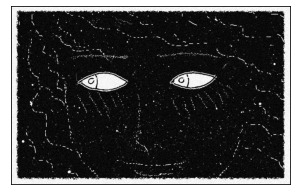}
    \caption{Face}
    \end{subfigure}
    \begin{subfigure}{0.32\textwidth}
    \centering
    \includegraphics[width=0.95\textwidth]{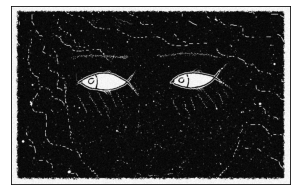}
    \caption{Fish}
    \end{subfigure}
    \begin{subfigure}{0.32\textwidth}
    \centering
    \includegraphics[width=0.95\textwidth]{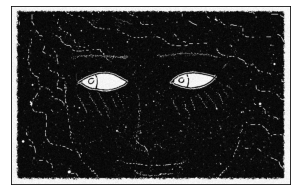}
    \caption{Face}
    \end{subfigure}
    \caption{Difference between human perception and $l_2$ distances illustrated by an optical illusion. {\bf (a)} The images looks like a face; {\bf (b)} The image looks like two fishes; {\bf (c)} The image in (a) adding a noise. The $l_0$, $l_1$ and $l_2$ distances between (a) and (b) are $15037$, $2534.44$ and $43.02$, respectively. The $l_0$, $l_1$ and $l_2$ distances between (a) and (c) are $812311$, $63413.43$ and $89.17$, respectively. Though the images in (a) and (c) are perceptually the same, their $l_p$ distances are greater than the distances between (a) and (b), which are perceptually different. (The original image is taken from https://pixabay.com/illustrations/fairy-tale-fish-portrait-1077859/)}
\label{fig: human_diff_lp}
\end{figure}

In this paper, we propose a conditional explanation of adversarial robustness, which highlights the role of human labeler in defining the adversarial examples. Specifically, we decouple the data generator with the labeler, and make two definitions: the {\bf teacher} is an object or a mechanism that assigns true labels to data points, and the {\bf student} is a machine learning model used to learn from the data and labels. Within our framework, adversarial robustness is not a universal concept defined unconditionally for any learning problem (like $l_p$ robustness), but rather a relative concept conditioned on a certain teacher. The teacher is usually human, but can also be other objects such as physical processes or neural networks. A student model is said to be (strongly) adversarially robust with respect to a teacher if it can correctly classify any data the teacher can classify with certainty. This is possible because in our framework the teacher has an ``uncertain set'', and it does not assign labels to data within this set. Hence a robust student model does not need to have the same decision boundary as the teacher. A weaker version of adversarial robustness is also defined by considering the data produced by an ``attack'', instead of all the data that the teacher can classify. This weak definition of adversarial robustness can cover the $l_p$ robustness, but in a more proper way. We show that our definitions of adversarial robustness are not equivalent with the $l_p$ robustness by simple illustrative examples---$l_p$ robust classifier may not be adversarially robust, vice versa. 

Based on this new understanding, we point out two reasons that cause adversarial examples:
(1) Some features the student uses to make classification are imperceptible to the teacher.
(2) The training data do not provide sufficient information of the classification mechanism of the teacher, e.g. which feature the teacher uses to make classification.
Combining the two reasons above, we argue that the adversarial examples are caused by insufficient (out-of-distribution) information of the teacher provided by the training data. Without necessary information, the student model cannot select the robust solution among many solutions that perform well on the original data distribution. Therefore, to achieve adversarial robustness, or at least alleviate adversarial vulnerability, more teacher information should be provided to the student model. This can be achieved in two ways:
\begin{enumerate}
    \item {\bf An active student:} The student model asks information from the teacher, and the teacher passively answers the student's questions, and does not provide extra information.
    \vspace{-5mm}
    \item {\bf An active teacher:} The teacher directly provides information to the student about how it makes classification.
\end{enumerate}
We show theoretically that the first way is not always efficient. Specifically, we prove that in some cases an active student cannot get enough information to achieve robustness in a reasonable time from a passive teacher. Hence, we conclude that an active teacher is required to achieve real adversarial robustness. By simple illustrative examples we show how an active teacher helps the student to learn a robust model, and better robustness can be achieved when more information is provided by the teacher.

Our contributions are summarized as follows:
\begin{itemize}
    \item We propose a new conditional framework of understanding adversarial robustness. In this framework, the teacher is decoupled from the distribution that generates the data, and robustness is defined as a relative concept of a student model with respect to the teacher.
    \item Based on the new understanding of adversarial robustness, we demonstrate that achieving robustness requires additional teacher information except the original training data. 
    \item Using both theoretical and empirical approaches, we show that an active teacher helps attaining robustness, while a passive teacher with an active student may not be as efficient.
\end{itemize}

\section{Related work}
Adversarial examples were first introduced in~\cite{szegedy2013intriguing}. The work identified data points that are very close to another point (imperceptible to human) but lead to totally different predictions of the model. Several attack methods were then proposed based on the idea of finding the direction in the input space in which the model's output changes fastest~\cite{goodfellow2014explaining, moosavi2016deepfool, papernot2016limitations, kurakin2016adversarial, kurakin2016adversarial2}. Due to the significance of the security of machine learning models, defenses for adversarial attacks also received extensive study (\cite{papernot2016distillation, wong2018provable, buckman2018thermometer, guo2017countering}, etc). Adversarial training~\cite{goodfellow2014explaining, kurakin2016adversarial2, madry2017towards, tramer2017ensemble} is a class of methods that can effectively defense against certain attacks. It trains a robust model by including adversarial examples into the training set. Large volume of works arise during an arm race between attacks and defenses. Interested readers can refer to~\cite{akhtar2018threat} or~\cite{hao2020adversarial} for a thorough review of the attack and defense methods in different application fields. 

On the other side of practical methods, theoretical understanding of  adversarial examples also drew attention. Explanations of adversarial vulnerability of machine learning models were provided from different perspectives, including linearity~\cite{goodfellow2014explaining}, decision boundary geometry~\cite{fawzi2016robustness, moosavi2017analysis}, low flexibility of the networks~\cite{fawzi2018analysis}, non-robust features~\cite{ilyas2019adversarial}, etc. In particular, \cite{ilyas2019adversarial} proposed that adversarial examples exist because the model learns non-robust features. This viewpoint can be put into our framework: non-robust features, though with good generalization performance, are not used by the teacher, the student cannot reject these features since the training data do not provide enough teacher information. 

Mathematical analysis were also conducted, e.g. to show the inevitable existence of adversarial examples~\cite{shafahi2018adversarial, fawzi2018adversarial}, the trade-off between adversarial robustness and clean data accuracy~\cite{tsipras2018robustness}, the trade-off between robustness and classifier complexity~\cite{nakkiran2019adversarial}, and the provable robustness of highly over-parameterized models~\cite{zhang2020over}. 

Due to its benefits on analysis and implementation, the $l_p$ distances are used to quantify robustness in most works mentioned above, especially the cases of $p=0,2,\infty$. However, $l_p$ distance is obviously different from human perception. In~\cite{engstrom2019exploring,zhao2020towards}, data pairs that are imperceptible to human but have large $l_p$ distances are identified. On the other side, \cite{sharif2018suitability} found image pairs that are close measured by the $l_p$ norm but look very different for humans. Attempts are made to find metrics that align better with human perception, such as the Wasserstein distance~\cite{wong2019wasserstein, wu2020stronger}, SSIM~\cite{gragnaniello2019perceptual} and other perceptibility metrics~\cite{laidlaw2020perceptual, jordan2019quantifying}. However, human experiments and statistical tests in~\cite{sen2019should} show significant difference between human perception and these metrics.

Among all the theoretical explanations of adversarial examples, the understanding provided in~\cite{wang2016theoretical} is most relevant to our work. Like what we do in this paper, the authors of~\cite{wang2016theoretical} also decouple the data generator and the label generator (which they call the oracle), and compare topological properties of the oracle and the student model. They claim that adversarial examples are caused by the difference of the two (pseudo)metric spaces corresponding to the student and the oracle. Our work is different from theirs in at least two ways: (1) After decoupling the data generator and the teacher, we directly compare the decision regions and decision boundaries of the teacher and the student, instead of considering metric spaces. The metric spaces help mathematical analysis, but are hard to verify and identify in practice. (2) Based on the decoupled understanding of adversarial examples, we further explore and compare possible ways to achieve adversarial robustness, and suggest that an active teacher is required to efficiently align student decision regions with those of the teacher in order to achieve adversarial robustness.

\section{A conditional framework of adversarial robustness}
\subsection{Decoupling data generator and teacher in supervised learning}

In this section we introduce a conditional framework to understand adversarial examples and adversarial robustness. We start from a decoupled understanding of supervised learning problems. Traditional formulation of supervised learning problems consists of two parts: a joint distribution of data and label $(\bx, y)$, and a student model which learns the relation between $\bx$ and $y$ using the training data sampled from the distribution. Compared with the traditional ones, our formulation of supervised learning decouples the process of generating $\bx$ and $y$, and consists of three components: the data generator, the teacher, and the student. 
\begin{itemize}
    \item {\bf The data generator} is a distribution $\mu$ from which data points $\bx$ are sampled, to form training and testing data sets.
    \item {\bf The teacher} is a mechanism to assign labels to the data points. It takes data $\bx$ as input and outputs a label $y$ associated with the data. The teacher can be a deterministic function or a stochastic mechanism. For practical machine learning problems the teacher is usually human. We use $T$ to denote the teacher.
    \item {\bf The student} is a machine learning model trained using a set of data and labels generated by the data generator and the teacher, to learn the labeling rules of the teacher. The student takes data points as inputs and the predicted labels for the input data as outputs. We use $S$ to denote the student.
\end{itemize}
Figure~\ref{fig: 1} shows the learning procedure of our machine learning model: the data generator generates data, the teacher assigns labels to the data, forming a dataset, and finally the student is trained using the dataset.

\begin{figure}[!h]
    \centering
    \includegraphics[width=0.7\columnwidth]{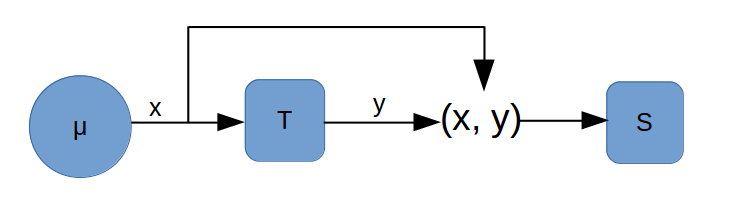}
    \caption{The learning procedure of the machine learning model considered in this paper.}
    \label{fig: 1}
\end{figure}

In our formulation we decouple the data generator and the teacher, so we can study the teacher alone. The decoupled perspective highlights that the teacher can work out of the data distribution $\mu$, and we do not have access to all the information of the teacher by just sampling data from $\mu$. As we are going to clarify in the next section, this is the essential reason for the existence of adversarial examples. Finally, note that the traditional formulation can also be included into our framework, by considering the data generator to be the marginal distribution (of $\bx$) and the teacher to be the conditional distribution (of $y$ conditioned on $\bx$).

\subsection{The conditional definition of adversarial examples}

By decoupling the teacher from the data generator, we can now examine adversarial examples and define adversarial robustness in a conditional way. Specifically, adversarial robustness is a property of a student model conditioned on a certain teacher. It involves both the student and the teacher.

We first express the ideas by a simple example. Assume $\bx\in X\subset\bR^d$. Consider a binary classification problem with two classes $A$ and $B$. Since sometimes a classifier cannot assign a label with high confidence for any $\bx$ in $X$, we assume that the teacher can output three values: $A$, $B$, and $U$. Here $A$ and $B$ mean the input data belongs to classes $A$ and $B$, respectively, and $U$ means the teacher is uncertain with the input data. This kind of classifiers are also studied as ``selective classifier'' in previous works~\cite{el2010foundations, geifman2017selective}. Let $\Omega_A\subset\bR^d$ be the set in which the teacher outputs $A$:
$$\Omega_A := \{\bx\in X: T(\bx)=A\}.$$
$\Omega_B$ and $\Omega_U$ are similarly defined. We require $\Omega_A\cup\Omega_B\cup\Omega_U=X$.

\begin{remark}
The existence of class $U$ is reasonable given that even for humans it is very common to be uncertain with some hard-to-classify images. We can understand the model as a classification problem with three classes but we are only interested in two of them. In traditional understanding of supervised learning the class $U$ is not highlighted because the data distribution is coupled with the teacher and naturally concentrates in $\Omega_A\cup\Omega_B$. But to address adversarial robustness the uncertain class $U$ becomes important because we have to consider adversarially generated unnatural data distributions. 
\end{remark}

\begin{remark}
The most interesting teachers are humans, which is the case for most CV and NLP problems. However, it can also be objects such as machine learning models, e.g. in the case of knowledge distillation~\cite{hinton2015distilling}. For an simple example, assume we have a neural network $N: \bR^d\rightarrow[0,1]$, which predicts the probability that the input belongs to class $A$. Then the teacher can be defined as 
\begin{equation*}
    T(\bx)=\left\{ \begin{array}{cc}
        A & \text{if}\ N(\bx)>0.99, \\
        B & \text{if}\ N(\bx)<0.01, \\
        U & \text{otherwise},
    \end{array}\right.
\end{equation*}
i.e. the classes $A$ and $B$ are assigned only when the neural network has high confidence. Hence, adversarial robustness can be considered with respect to general teachers, as in the examples below.
\end{remark}

Now we can give a formal description of adversarial examples within our framework. Usually an adversarial example is defined as a data point wrongly classified by a machine learning model, which is very close to another correctly classified data point, and the difference between the two data points are imperceptible to humans. In our framework, we let humans be the teacher and the machine learning model be the student. We highlight the fact that the student gives different prediction from the teacher, then the above definition of the adversarial examples can be rephrased as follow:
\begin{center}
    {\it An adversarial example is a data point $\bx$ that satisfies $T(\bx)=A\ \textrm{or}\ B$ and $T(\bx)\neq S(\bx)$.}
\end{center}
Later examples will show that, as long as adversarial examples described above exist, there will naturally be adversarial examples perceptually close to a correctly classified data point.

In the above statement, an adversarial example can be understood as a data point that the teacher can classify with high confidence, but the student gives different label from the teacher. This kind of data exists because the student is trained by data sampled from $\mu$, but $\mu$ cannot provide full information of the teacher, e.g. the support of $\mu$ cannot fully cover $\Omega_A$ and $\Omega_B$. As an illustrative example, (See the left panel of Figure~\ref{fig: 2}) let $\bx=(x_1,x_2)\in[-1,1]^2$ and the teacher is induced by a linear model:
\begin{equation*}
    T(\bx)=\left\{\begin{array}{cl}
        A & \text{if}\ x_1\geq0.5, \\
        B & \text{if}\ x_1\leq-0.5, \\
        U & \text{otherwise}.
    \end{array}\right.
\end{equation*}
On the other side, assume that $\mu$ is a uniform distribution on $[0.5,1]\times[0.5,1]\cup[-1,-0.5]\times[-1,-0.5]$. Then if the student makes max margin classification, the decision boundary will be close to $x_1+x_2=0$. Adversarial examples appear in the second and fourth quadrant (as show by the grey areas in the figure). 

\begin{figure}[!h]
    \centering
    \begin{subfigure}{0.45\textwidth}
    \centering
    \includegraphics[width=0.95\textwidth]{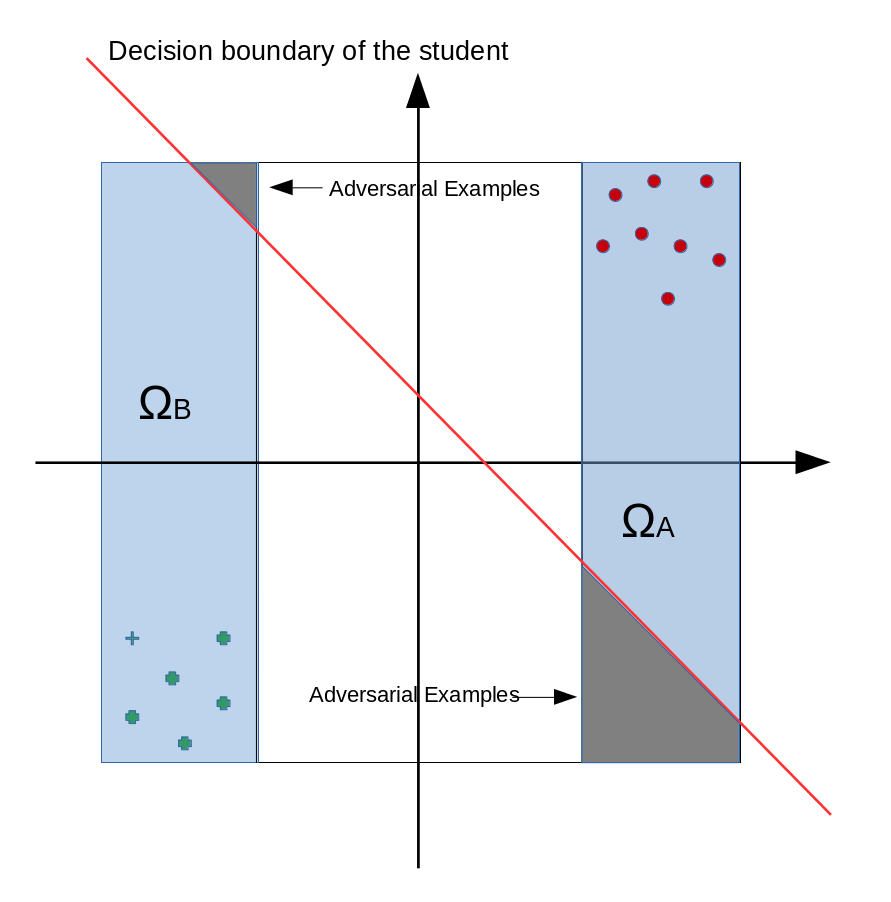}
    \caption{}
    \end{subfigure}
    \begin{subfigure}{0.45\textwidth}
    \centering
    \vspace{5mm}
    \includegraphics[width=0.83\textwidth]{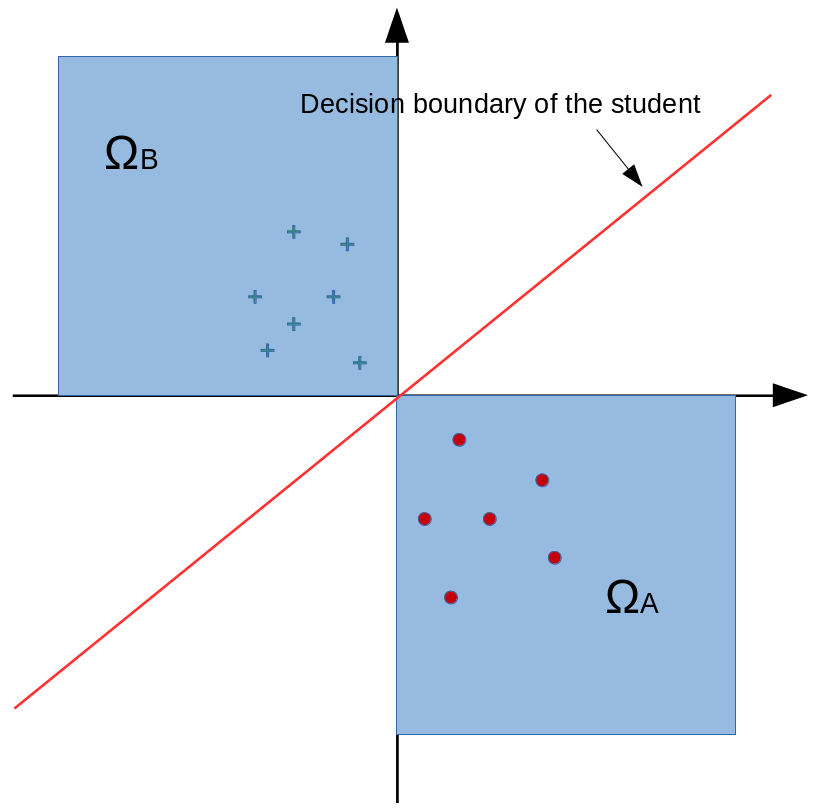}
    \caption{}
    \end{subfigure}
    \caption{{\bf (Left)} The simple illustration of adversarial examples. $l_2$ robust model may not be adversarially robust. {\bf (Right)} Strongly adversarially robust model may not be $l_2$ robust..}
    \label{fig: 2}
\end{figure}

With the above definition of adversarial examples, we can state the following definition of adversarial robustness:

{\it A student model $S$ is adversarially robust with respect to a teacher $T$, if $T(\bx)=S(\bx)$ for all $\bx\in\Omega_A\cup\Omega_B$.}

We call this definition {\bf Strong Adversarial Robustness}, because it requires the student to generalize on any distribution in $\Omega_A\cup\Omega_B$, i.e. it should give correct classification on any data point that the teacher can classify with high confidence. It is clear that in this situation no perturbation imperceptible to the teacher can lead to a change of classification of the student. Note that strong adversarial robustness does not require the student to be the same as the teacher, due to the existence of $U$. 

The definition of strong adversarial robustness can be extended to multi-class classification problems. Assume there are $K\geq2$ classes denoted by $C_1, C_2, ..., C_K$, and let $\Omega_{C_k}$ be the regions where $T$ outputs $C_k$, for $k=1,2,...,K$. Then, we have the following definition for strong adversarial robustness:
\begin{definition}{(Strong adversarial robustness)}
Let $\mu$, $T$, $S$ be the data distribution, teacher, and student, respectively. Then, $S$ is strongly adversarially robust with respect to $T$ if 
\begin{equation*}
    S(\bx)=T(\bx),\ \forall \bx\in\bigcup_{k=1}^K\Omega_{C_k}.
\end{equation*}
\end{definition}

Besides strong adversarial robustness, we can also define a weaker version of adversarial robustness. In this case, we consider an attack $\cA$ which takes the original data distribution and the student model as inputs and a family of adversarial data distributions as output. We say the student $S$ is adversarially robust with respect to the teacher $T$ and the attack $\cA$ if $S$ generalizes as well as $T$ on all the distributions generated by the attack $\cA$. A mathematical definition is given as follows.
\begin{definition}{(Adversarial robustness with respect to an attack)}\label{def: weak}
Let $\mu$, $T$, $S$ be the data distribution, teacher, and student, respectively. Let $\cA$ be the attack, and 
$$\cP=\cA(\mu,S),$$
where $\cP$ is a family of adversarial distributions given by $\cA$ with input $\mu$ and $S$. Then, the student $S$ is $(\cA,\mu,\epsilon)$-adversarially robust with respect to $T$, if
\begin{equation}\label{eqn: weak_adv}
    \inf_{\nu\in\cP}\bP_{\bx\sim\nu}\left(T(\bx)=S(\bx)|\bx\in\bigcup_{k=1}^K\Omega_{C_k}\right) > 1-\epsilon.
\end{equation}
\end{definition}
By the definition above, the student is adversarially robust if it can generalize well over the distributions generated by a specific attack, on the regions where the teacher performs with certainty. The attack can take many forms. For example, the attack with small $l_p$ perturbations produces all the distributions whose support is within a small distance $\delta$ of the support of $\mu$:
\begin{equation*}
    \cP = \left\{ \nu: \forall \bx\in\supp(\nu),\ \exists\bx'\in\supp(\mu),\ s.t.\ \|\bx-\bx'\|_p\leq\delta \right\}.
\end{equation*}
Note that the weak adversarial robustness with above attack is not exactly equivalent with the commonly studied $l_p$ robustness. Because in our definition we only require the student to classify correctly in the region where the teacher can make confident classification, instead of giving the same classification within the $l_p$ ball with radius $\delta$ centered at any data point $\bx$. (See the conditional probability in~\eqref{eqn: weak_adv}) As a results, our definition of adversarial robustness does not conflict with the clean data accuracy (the accuracy on $\mu$)---the student can be robust at the same time of having good accuracy on $\mu$. This is a more proper definition of $l_p$ robustness.  

As a second example, the attack can also be all the distributions whose Radon-Nikodym derivative with respect to $\mu$ is close to $1$:
\begin{equation*}
    \cP = \left\{\nu: \frac{d\nu}{d\mu}\in[c, \frac{1}{c}]\right\},
\end{equation*}
for some constant $c>0$. As a third example, it can also depend on the student $S$, such as the fast gradient method:
\begin{equation*}
    \cP = \left\{ \nu=\Gamma\#\mu:\ \Gamma(\bx) = \bx+\delta\frac{\bg}{\|\bg\|},\ \bg=\frac{\partial S(\bx)}{\partial\bx} \right\}.
\end{equation*}
Finally, strong adversarial robustness can be viewed as robustness with an attack that produces all the probability distributions on $X$.

\subsection{Relation with $l_p$ robustness}

As we mentioned above, $l_p$ robustness is appropriately covered by Definition~\ref{def: weak}. In this section, we focus on traditional $l_2$ robustness and compare it with our definition of strong robustness. Using simple illustrative examples, we show that $l_2$ robust students may not be strongly robust, and strongly robust students may not be $l_2$ robust, either. 

The example in the left panel of Figure~\ref{fig: 2} shows a student that is $l_2$ robust but not strongly adversarially robust with respect to the teacher. In the example, the teacher conducts classification with only $x_1$, and does not use the feature $x_2$. Hence, data points with the same $x_1$ but different $x_2$ are imperceptible to the teacher. However, the student gathers teacher information only from the training data, hence it is reasonable for it to make max margin classification. Using the max margin decision boundary, the student is $l_2$ robust even when the perturbation is large, but adversarial examples exist. For example, for a data point in the grey area on the upper-left part of the figure, the student will make wrong classification, while for the teacher this data a looks similar to the ones on the bottom-left side because they have the same $x_1$.

In the right panel of Figure~\ref{fig: 2}, we show an example that adversarial robustness does not imply $l_2$ robustness. In this example, the two classes lie in the second and the fourth quadrants, respectively. And there is no margin between the two classes. The student with decision boundary shown by the red line is strongly adversarially robust with respect to the teacher, because the decision boundary passes through the origin. However, since there is no margin between $\Omega_A$ and $\Omega_B$, the student is not $l_2$ robust with any $\epsilon>0$, because for any $\epsilon$ we can always find a sample in $\Omega_A\cup\Omega_B$ whose $l_2$ distance from the decision boundary is smaller than $\epsilon$. In real problems, such ``zero margin'' situation is quite common. The teacher's decision might have sudden jumps from one class to another in a small region, for example, when the teacher decides the sign of a number, or compares the sizes of two objects. Humans are usually good at these tasks.

\section{Adversarial robustness requires active teacher}

By the new understanding of adversarial examples, we tentatively conclude that adversarial examples exist because the student does not have sufficient information of the teacher. To achieve adversarial robustness, additional teacher information should be incorporated into the student. This can be achieved in two ways:
\begin{enumerate}
    \item {\bf A passive teacher and an active student:} In this approach, the teacher provides information to the student only when the student asks for information from the teacher. For instance, in addition to the training data, the student generates extra data and asks the teacher to classify these data. Then, the new data and labels are included to the training set to train an updated student model. (This is like an interactive way of adversarial training). 
    \vspace{-2mm}
    \item {\bf An active teacher:} The teacher directly tells the student information on how it makes classification, such as the features used, invariances, sparsity, or the structure of the model, etc. Then, the student tries to encode the information into its learning procedure, e.g. taking specially designed network structure and learning algorithm.
\end{enumerate}
In this section, we show that an active teacher is preferred, and may even be necessary, for the student to be adversarially robust. 

In the setting of a passive teacher, we theoretically prove that a simple query-based active student cannot efficiently learn robust models. On the other hand, in the setting of an active teacher, we show by numerical examples how can the teacher ``teach'' the student to be robust.

\subsection{A passive teacher and an active student}

An active student can acquire teacher information in many different ways. In this section, we consider one of the most natural ways to ask for teacher information---feature query. Specifically, every time the student provides the teacher with a feature, and the teacher returns the correlation of the feature with the labels (the correlation is computed in a data distribution generated by the attack, hence it helps achieving adversarial robustness and cannot be approximated with $\mu$). In this way, the student asks the teacher ``to what extent do you use this feature to make classification'', and the teacher answers the question with a score. Then, the student updates itself according to the teacher's answer. Intuitively, the student can learn a robust classifier if it identifies all the features used by the teacher to make classification. However, since there are numerous possibilities when choosing the features to query, it can be hard to find the right ones. In this section, we borrow the theories of hardness of learning to show that in some cases it is impossible to efficiently learn a robust student with feature querying, even though we have a very weak attack which only produces one single adversarial distribution.

Mathematically, we put our ``feature query'' setting into the statistical query framework~\cite{kearns1998efficient}. Let $X=\{0,1\}^d$, $\cD$ be some probability distribution on $X$. Let $T: X\rightarrow\{0,1\}$ be the teacher. Then, a statistical query $\stat(T,\cD)$ takes a function $\chi: X\times\{0,1\}\rightarrow\{0,1\}$ and returns $\bE_{\bx\sim\cD}\chi(\bx,T(\bx))$ with some tolerance $\alpha$, i.e. the returned value lies in $[\bE_{\bx\sim\cD}\chi(\bx,T(\bx))-\alpha, \bE_{\bx\sim\cD}\chi(\bx,T(\bx))+\alpha]$. Obviously, the correlation of a feature $h: X\rightarrow\{0,1\}$ with the labels, $\bE_{\bx\sim\cD} h(\bx)T(\bx)$, is a statistical query with $\chi(\bx, T(\bx))=h(\bx)T(\bx)$. Statistical queries are powerful because it can return the correlation of any feature with the teacher's output with high accuracy, of course including those features used by the teacher.

It is proven in~\cite{kearns1998efficient} that parity functions are not efficiently learnable from statistical queries:

\begin{theorem}\label{thm: parity}
{\bf (Theorem 5 of~\cite{kearns1998efficient})}
Let $\cF_d$ be all parity functions over $\{0,1\}^d$ and $\cD$ be the uniform distribution on $\{0,1\}^d$. Then, for any fixed accuracy $\varepsilon$, there does not exist polynomials $p(d)$ and $q(d)$, and an algorithm using statistical queries with tolerance $\alpha\geq1/q(d)$, such that for any $f\in\cF_d$ the algorithm can return a hypothesis $h$ within $p(d)$ statistical queries that satisfies
\begin{equation}
    \bP_{\bx\sim\cD}(h(\bx)=f(\bx)) > 1-\varepsilon.
\end{equation}
\end{theorem}

Based on the theorem above, we can show that an active student using feature queries cannot always learn adversarially robust classifiers efficiently. Still consider $X=\{0,1\}^d$. Now, let $\mu$ be the uniform distribution on two points $(0,0,...,0)$ and $(1,0,...,0)$, $\cA$ be an attack, and $\nu$ be the uniform distribution on $X$, which is generated by the attack $\cA$. That is to way, the output of this weak attack contains only one distribution, and even does not depend on the student. Finding a robust classifier requires the student to generalize on $\nu$. Consider the set of teachers $\cT$ to be all parity functions over $X$ with the first coordinate included, i.e.
\begin{equation}
    \cT = \left\{T(\bx)=\frac{1}{2}\left(1+(-1)^{\sum_{i\in S}\bx_i}\right):\ S\subset[d], 1\in S\right\}.
\end{equation}
Then, since $\mu$ only supports on two points, teachers in $\cT$ can be learned by a simple linear regression on $\mu$. However, by Theorem~\ref{thm: parity}, they cannot be learned efficiently on $\nu$ using feature queries. Hence, the student cannot learn adversarially robust classifiers with respect to the teachers in $\cT$, if the attack gives the distribution $\nu$. To summarize, we have the following theorem.

\begin{theorem}\label{thm: 2}
  Let $\cT$, $\mu$, $\nu$, $\cA$ be defined above. Let $T$ be a teacher from $\cT$ and $S$ be a
  student which has access to the data pairs $(\bx, T(\bx))$ where $\bx$ is sampled from
  $\mu$. Besides, the student can get feature queries $\bE_{\bx\sim\nu} h(\bx)T(\bx)$ for any
  feature $h:X\rightarrow\{0,1\}$, with a tolerance $\alpha$ that satisfies $\alpha\geq 1/q(d)$ for
  some polynomial $q(d)$. Then, for any fixed $\varepsilon>0$, there does not exist a polynomial
  $p(d)$ such that for any $T\in\cT$ the student can learn an $(\cA,\mu,\varepsilon)$-Adversarially
  robust classifier with respect to $T$ within $p(d)$ feature queries.
\end{theorem}

\begin{proof}
For any $d\geq 2$, let $X_0 = \{(0, x_2, x_3, ..., x_d):\ x_i\in\{0,1\},\ i=2,3,...,d\}$. Assume
that the conclusion of Theorem~\ref{thm: 2} does not hold. Then, there exists an algorithm that for
any teacher $T\in\cT$ it can learn a student model $h$ that satisfies
\begin{equation}\label{eqn: proof1}
  \bP_{\bx\sim\cD}(h(\bx)=T(\bx)) > 1-\epsilon
\end{equation}
with at most $p(d)$ feature queries and a tolerance $\alpha\geq1/q(d)$. Here, $\epsilon$ is a constant and $p(\cdot), q(\cdot)$ are
two polynomials, which may depend on $\epsilon$. Equation~\eqref{eqn: proof1} implies
\begin{equation*}
    \bP_{\bx\sim\cD}(h(\bx)\neq T(\bx)) \leq \epsilon,
\end{equation*}
which can be rewritten as 
\begin{equation*}
    \frac{1}{2^d}\sum\limits_{\bx\in\{0,1\}^d}\mathbf{1}_{h(\bx)\neq T(\bx)} \leq \epsilon.
\end{equation*}
Therefore,
\begin{equation*}
    \sum\limits_{\bx\in X_0} \mathbf{1}_{h(\bx)\neq T(\bx)} \leq \sum\limits_{\bx\in\{0,1\}^d}\mathbf{1}_{h(\bx)\neq T(\bx)} \leq 2^d\epsilon,
\end{equation*}
which directly gives 
\begin{equation*}
    \frac{1}{2^{d-1}}\sum\limits_{\bx\in X_0} \mathbf{1}_{h(\bx)\neq T(\bx)} \leq 2\epsilon,
\end{equation*}
and hence
\begin{equation}\label{eqn: proof2}
    \bP_{\bx\sim \textrm{unif}(X_0)}(h(\bx)=T(\bx)) > 1-2\epsilon.
\end{equation}
By the definition, $\cT$ conditioned on $X_0$ contains all the parity functions of $x_2, ..., x_3$. Hence, Equation~\eqref{eqn: proof2} is contradictory with Theorem~\ref{thm: parity}. This completes the proof. 
\end{proof}

The hardness of learnability has recently been studied in the setting of neural networks~\cite{livni2014computational,goel2020superpolynomial,goel2020statistical}. Therefore, it is possible to extend Theorem~\ref{thm: 2} to more general teachers, e.g. neural network models.

\subsection{An active teacher}\label{ssec: active_teacher}

On the other hand, if the teacher actively provides information to the student, then it is possible to efficiently learn robust classifiers. For the same problem in Theorem~\ref{thm: 2}, if the student knows from the teacher that it is a parity function, then the student can check whether the teacher considers the $i$-th coordinate by querying two data points $(0,0,...,0)$ and $(0,...,0,1,0,...,0)$ where the $1$ in the second data point appears in the $i$-th coordinate. In this way, the student can learn the teacher within $d+1$ data queries. Hence, we have the following theorem:
\begin{theorem}\label{thm: 3}
Let $\cT$, $\mu$, $\nu$, $\cA$ be defined the same as in Theorem~\ref{thm: 2}. The teacher $T$ comes from $\cT$. Let $S$ be a student that can make data query from the teacher, i.e. get $T(\bx)$ from the teacher for any $\bx$. Then, if the student knows the teacher is a parity function, it can learn a strongly adversarially robust classifier within $d+1$ data queries.
\end{theorem}

\begin{proof}
Let $\mathbf{0}=(0,0,...,0)$, and $\be_i=(0,...,0,1,0,...,0)$ for $i=1,2,...,d$, where the $1$
appears on the $i$-th coordinate. If the student knows that the teacher $T$ comes from parity
functions, it can query $T(\mathbf{0})$ and $\{T(\be_i)\}_{i=1}^d$. Then for any $i=1,2,...,d$,
$T(\mathbf{0})=T(\be_i)$ implies $i\not\in S$, while $T(\mathbf{0})\not=T(\be_i)$ implies $i\in S$.
\end{proof}

Therefore, information directly from an active teacher may help the student find a robust classifier more efficiently. In the following we support this claim by two numerical examples.

\paragraph{Example 1.} Consider a binary classification problem. Let $\bx\in\bR^{100}$ be the input data, and $\bx_{i}$ be the i-th element of $\bx$. Assume $\bx_{i}\in[0,1]$. To assign the label, the teacher only compares the first element $x_{1}$ and the last element $x_{100}$. The teacher assigns label $y=-1$ if $x_{1}>x_{100}$, and $y=1$ if $x_{1}\leq x_{100}$. Obviously a strongly adversarially robust classifier for this problem cannot be $l_2$ robust, because there is no margin between the two classes.
For each class, we uniformly sample $1000$ training data. Linear regression (without bias) is used as the student model. Let the linear regression model be
\begin{equation}
    \hat{y} = \alpha^T\bx,
\end{equation}
and let $\alpha_{i}$ be the coefficients corresponding to $\bx_{i}$. Then the strongly adversarially robust model satisfies $\alpha_1<0$,  $\alpha_{100}=-\alpha_{1}$, and $\alpha_{i}=0$ for other $i$. 

If the student does not have any additional information besides the training data, a plain linear regression is conducted with $100$ variables. A dense vector will be produced, and it will be easy to find adversarial examples by changing $\bx_{i}$'s other than $\bx_{1}$ and $\bx_{100}$ according to the sign of the corresponding coefficients. Specifically, for some $\bx$ correctly classified by the student, assume $\bx_1>\bx_{100}$ without loss of generality, we can construct $\hat{\bx}$ by
\begin{equation*}
    \hat{\bx}_1=\bx_1,\ \hat{\bx}_{100}=\bx_{100},\ \hat{\bx}_i = \bx_i+\epsilon\textrm{sign}(\alpha_i)\ \textrm{for}\ i=2,3,...,99.
\end{equation*}
Then, the prediction of $\hat{\bx}$ can be flipped as long as 
$\epsilon > \alpha^T\bx/\sum_{i=2}^{99}|\alpha_i|$,
while the difference between $\hat{\bx}$ and $\bx$ is always imperceptible to the teacher.
Figure~\ref{fig: 5} shows the coefficients and some adversarial examples in the form of $10\times10$ images. On the other hand, if the student is provided with additional information directly from the teacher beyond the training data, better adversarial robustness may be achieved. For example, if the student is told that the teacher only considers $\bx_{1}$ and $\bx_{100}$, then the student can choose to use a sparse model
\begin{equation}
    \hat{y} = \alpha_1\bx_{1} + \alpha_2\bx_{100}.
\end{equation}
Training this sparse model with the same set of training data, we obtain the model
\begin{equation}\label{eqn: sparse_model}
    \hat{y} = -2.01\bx_{1} + 1.98\bx_{100},
\end{equation}
which is much more robust than the plain linear regression model, because perturbing pixels other than $\bx_{1}$ and $\bx_{100}$ can no longer change the prediction of the model. However, adversarial examples still exist for those $\bx_{1}, \bx_{100}$ that satisfies $-\bx_{1}+\bx_{100}>0$ but $-2.01\bx_{1} + 1.98\bx_{100}<0$, e.g. $\bx_1=1,\ \bx_{100}=1.01$. If we incorporate further information, e.g. the teacher is a linear model that takes integer coefficients, then we can round the coefficients in~\eqref{eqn: sparse_model} and get a model with strong adversarial robustness.

\begin{figure}
    \centering
    \includegraphics[width=0.18\textwidth]{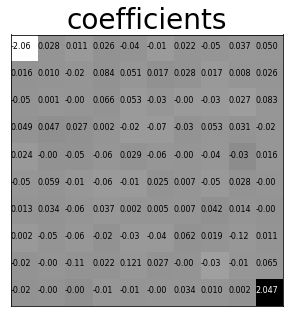}
    \includegraphics[width=0.18\textwidth]{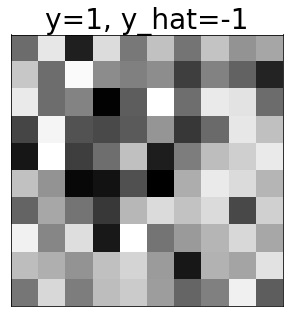}
    \includegraphics[width=0.18\textwidth]{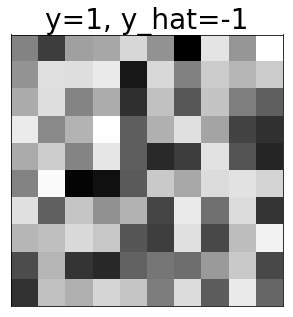}
    \includegraphics[width=0.18\textwidth]{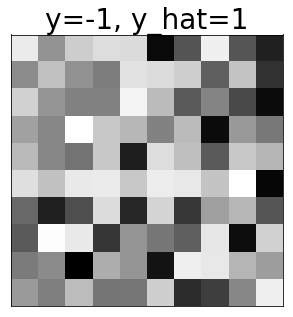}
    \includegraphics[width=0.18\textwidth]{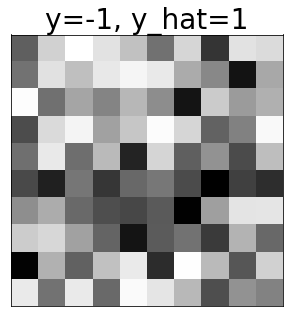}
    \caption{({\bf Left}) the coefficients of the student model corresponding to every entries of $\bx$. ({\bf Others}) some adversarial examples of the student model.}
    \label{fig: 5}
\end{figure}

\paragraph{Example 2.} This example shows that when multiple features can be picked to make generalizable classification, additional teacher information can help the student find the features that lead to a robust model. In this problem, the data are images of a disk or a square with random size and location, and the student model is asked to classify between disks and squares. Except the shapes, we add textures in the squares as a confounding feature. Examples of the data are shown on the first row of Figure~\ref{fig: round_square}. 
In the data distribution which generates the training and testing data ($\mu$), squares always have textures while disks always do not. Therefore, both features---shape and texture---can be used to build a generalizable classifier. However, for the teacher (human) shape and texture have different meanings and the teacher expect the student to use shape for classification. Hence, as an adversarial data distribution ($\nu$), we generate disks with texture and squares without texture, as shown on the second row of Figure~\ref{fig: round_square}. 

\begin{figure}
    \centering
    \begin{subfigure}{0.96\textwidth}
    \centering
    \includegraphics[width=0.18\textwidth]{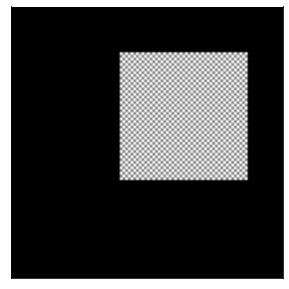}
    \includegraphics[width=0.18\textwidth]{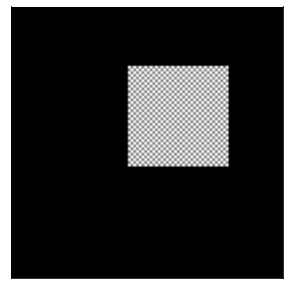}
    \includegraphics[width=0.18\textwidth]{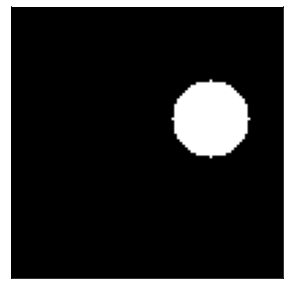}
    \includegraphics[width=0.18\textwidth]{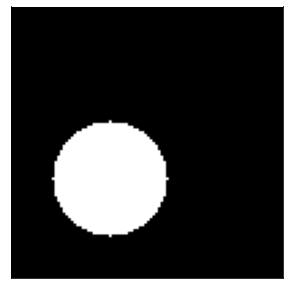}
    \caption{}
    \end{subfigure}
    \begin{subfigure}{0.96\textwidth}
    \centering
    \includegraphics[width=0.18\textwidth]{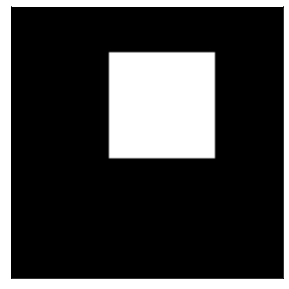}
    \includegraphics[width=0.18\textwidth]{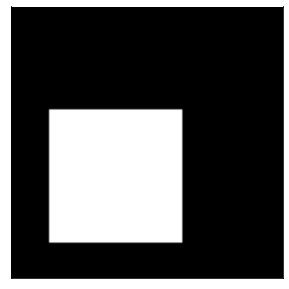}
    \includegraphics[width=0.18\textwidth]{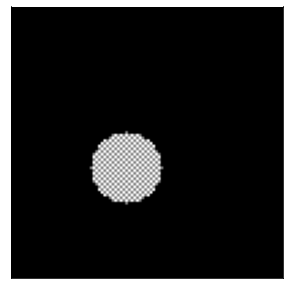}
    \includegraphics[width=0.18\textwidth]{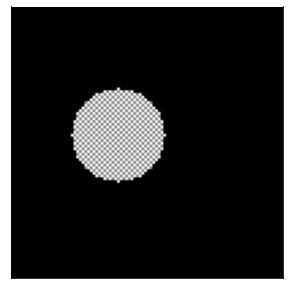}
    \caption{}
    \end{subfigure}
    \caption{(a) Examples of the training and testing data. The location and size of the squares and disks are chosen randomly. The squares have textures while the disks do not. (b) Examples of the data generated by the attack from adversarial distribution $\nu$. Textures appear in the disks instead of the squares..}
    \label{fig: round_square}
\end{figure}

A convolutional neural network is utilized to learn the problem on a training set including $1000$ images, with $500$ squares and $500$ disks. Experiment details are provided in the appendix. $1000$ test samples are randomly generated from $\mu$, and another $1000$ adversarial examples are generated from $\nu$. The left panel of Figure~\ref{fig: round_square_res} shows the accuracy on the test samples and adversarial samples during the training process, when no teacher information except the training data is provided. It shows clearly that the student learns to make classification using textures, hence as the test accuracy goes to $1$ the adversarial accuracy goes to $0$. On the other hand, if the teacher tells the student that the classification should be made depending on the shape, then the student can conduct a low-pass filtering to the images before feeding them into the neural network, to filter out the texture. For this problem, specifically, we use a max pooling with kernel size $3$ and stride $1$ to act as the filtering. The results are shown on the right panel of Figure~\ref{fig: round_square_res}. In this case the adversarial accuracy is nearly as good as the test accuracy. 

\begin{figure}
  \centering
  \includegraphics[width=0.4\textwidth]{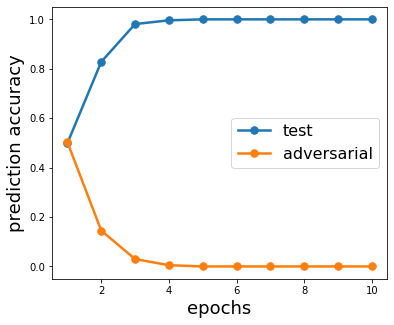}
  \includegraphics[width=0.4\textwidth]{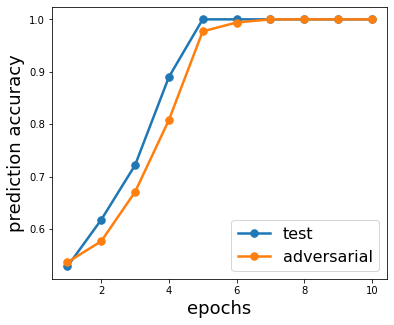}
  \caption{The prediction accuracy on test samples and adversarial samples. ({\bf Left}) The training data are directly used to train the neural network. ({\bf Right}) A max pooling of kernel size $3$ and stride $1$ is conducted before the data is fed into the convolution layers.}
  \label{fig: round_square_res}
\end{figure}

\section{Discussion}
In this paper we make three points about the cause of and the solution to adversarial examples. First, the teacher and the data generator should be considered separately, and adversarial robustness is a relevant concept between the student and the teacher. Second, adversarial examples are caused by the insufficiency of information provided by the training data about the teacher. Third, to solve the insufficiency of teacher information, we suggest that an active teacher is more preferred than an active student with a passive teacher. In the case where the teacher is human, our study suggests that human labelers should provide more information besides the labels and the model should be designed to incorporate the additional information. This is similar to the case when people are learning. For example, when human teachers teaches image recognition to human students, they usually describe features about the objects. The description of features certainly contains more information than just labels. 

Moreover, in complicated learning problems the features are often hierarchical. In deep learning, one often prefers end-to-end training and relies on the models to automatically learn the hierarchical structure of the features. Our study, however, demonstrates that including information of the feature hierarchy may help the student model be robust. Similar methodology has been studied in a different context. In~\cite{shalev2017failures}, it is shown that decomposition learning can be efficient when end-to-end learning is impossible. 

Strictly speaking, any model or algorithm encodes certain prior information and hence exhibits certain ``implicit bias''. The model performs well when its implicit bias coincides with the prior of the teacher. This is especially crucial in the over-parameterized regime where there are many solutions which perfectly fit the training data but only a small fraction of them generalize well. In the case of adversarial robustness, however, we require another level of implicit bias: the solutions picked by the model not only have to generalize well on the test data provided by the data generator, but also need to generalize to regions that are not sufficiently represented by the data generator. This is also a topic studied by out-of-distribution generalization~\cite{sun2019test,krueger2020out} and distribution shift~\cite{quionero2009dataset}. However, existing models cannot provide satisfactory implicit bias to learn human-like classifiers. They are either too simple (like the linearity of linear regression and the sparsity of LASSO), or hard to interpret (like deep neural networks). It is very important to design models that can directly and explicitly incorporate interpretable information provided by humans. We leave this as a major direction of future work. 

Finally, other than achieving robustness, a model whose prior knowledge is better aligned with that of humans may also help in few-shot learning, meta-learning, and model interpretation. It is an inevitable step to achieve higher levels of artificial intelligence than today's deep learning.

{
\bibliographystyle{plain}
\bibliography{adv_robustness}
}

\newpage
\appendix

\section{Experiment details}
In this section we provide the experimental details of the second example in Section~\ref{ssec:
  active_teacher}. The experiments are run on a 2020 Macbook Pro 13' with 16GB RAM, and the neural
networks are implemented and trained by Pytorch.

\subsection{Data}
The data are images with $100\times100$ pixels. The half side length of the squares and the radius
of the disks are uniformly sampled from integers within $[12,25]$, which roughly corresponds to
$[a/8, a/4]$, where $a$ is the side length of the images. The centers of the shapes are then
uniformly sampled from pixels so that the whole shape is within the image. For example, for a square
with half side length 20, the center is a pixel $(x,y)$ with $x$ and $y$ uniformly sampled from
integers within $[20, 78]$. Then, the area of the square is $[x-20,x+20]\times[y-20,y+20]$. The
images are in gray scale, with each pixel taking values in $[0,1]$. The background pixels take
values $0$ while the pixels in the shape are $1$. The texture is added to the shape by changing the
value from $1$ to $0.5$ for the pixels $(x,y)$ with $x+y$ being even and leaving the value at other
pixels unchanged.

For training, we sample $500$ images of squares with texture and $500$ images of disks without
texture, forming a data set consisting of $1000$ images. For testing, we sample $1000$ new image
each time testing is conducted. The images still consist of squares with textures and disks without
textures. The probability of squares is $0.5$. When measuring adversarial performance, each time we
sample $1000$ images of squares without texture and disks with texture. The probabilities of squares
and disks are still $0.5$.

\subsection{Model}
We use a multi-layer convolutional neural network (CNN) as the student model. The neural network has
$3$ convolution layers and $2$ max pooling layers in the middle of convolution layers. An average
pooling and a fully connected layer follow the convolution layers. Specifically, the architecture of
the neural network is

{\centering
Input ($100\times100\times1$)\\
$\downarrow$\\
Convolution with $16$ channels ($100\times100\times16$)\\
$\downarrow$\\
Max pooling the kernel size 2 and stride 2 ($50\times50\times16$)\\
$\downarrow$\\
Convolution with $32$ channels ($50\times50\times32$)\\
$\downarrow$\\
Max pooling the kernel size 2 and stride 2 ($25\times25\times32$)\\
$\downarrow$\\
Convolution with $32$ channels ($25\time25\times32$)\\
$\downarrow$\\
Average pooling the kernel size 5 and stride 5 ($5\times5\times32$)\\
$\downarrow$\\
Reshape and fully connected layer ($2$)\\
}
\ \\

The cross entropy loss is used as the loss function. Adam is taken as the optimizer, with learning
rate $0.001$ and default momentum factors $(0.9,0.999)$. The network is trained by $10$ epochs and
the batch size is $50$.

\end{document}